\documentclass[twocolumn]{article}

\usepackage[american]{babel}
\usepackage{natbib} 
\usepackage{mathtools} 
\usepackage{booktabs} 
\usepackage{tikz} 
\usepackage{subcaption}
\usepackage{authblk}
\usepackage[linesnumbered,ruled,vlined]{algorithm2e}
\usepackage{hyperref}

\newtheorem{theorem}{Theorem}
\newtheorem{corollary}{Corollary}

\newtheorem{definition}{Definition}
\newtheorem{proof}{Proof}

\title{Choosing DAG Models Using Markov and Minimal Edge Count in the Absence of Ground Truth}

\author[1]{Joseph D. Ramsey}
\author[2]{Bryan Andrews}
\author[1]{Peter Spirtes}
\affil[1]{Carnegie Mellon University Philosophy Department }
\affil[2]{University of Minnesota Department of Psychiatry \& Behavioral Sciences }
  
\begin{document}

\maketitle

\begin{abstract}
We give a novel nonparametric pointwise consistent statistical test (the Markov Checker) of the Markov condition for directed acyclic graph (DAG) or completed partially directed acyclic graph (CPDAG) models given a dataset. We also introduce the Cross-Algorithm Frugality Search (CAFS) for rejecting DAG models that either do not pass the Markov Checker test or that are not edge minimal. Edge minimality has been used previously by Raskutti and Uhler as a nonparametric simplicity criterion, though CAFS readily generalizes to other simplicity conditions. Reference to the ground truth is not necessary for CAFS, so it is useful for finding causal structure learning algorithms and tuning parameter settings that output causal models that are approximately true from a given data set. We provide a software tool for this analysis that is suitable for even quite large or dense models, provided a suitably fast pointwise consistent test of conditional independence is available. In addition, we show in simulation that the CAFS procedure can pick approximately correct models without knowing the ground truth.
\end{abstract}

\section{Introduction}\label{sec:intro}

Directed acyclic graph (DAG) data models are often published, making modeling claims for which little justification exists, other than that they are produced by a ``reliable'' algorithm. Claims about ``reliability'' are typically justified in two different ways: theorems that state that under suitable assumptions the algorithm is correct in the large sample limit and/or simulation studies. However, both types of reliability indicator may not be applicable to a particular algorithm or search parameter choice on given data. Simulations routinely make assumptions that real data violate (e.g., no latent confounders, no selection bias, no measurement error, i.i.d. sampling, etc.). Real data may require, for correctness, assumptions that are not simultaneously satisfied by \textit{any} existing algorithm, and when these implicit assumptions are violated, errors can occur. Accordingly, algorithms may only be reliable in certain contexts, primarily simulation contexts, in aggregate, and perhaps not in a context suitable for analyzing a specific dataset. 

In the case of theorems guaranteeing the correctness of causal structure learning algorithms, the theorems typically either only guarantee correctness in the large sample limit \citep{spirtes2000causation, chickering2002optimal}, or they make implausibly strong assumptions \citep{kalisch2007estimating, uhler2013geometry}, or at finite sample sizes they guarantee only very loose bounds on the probability of being correct (Spirtes and Zhang, Spirtes and Wang). Correctness in the large-sample limit does not guarantee success at any given finite sample size. Thus, in many (but not all) cases, violations of assumptions and/or finite sample sizes lead algorithms to produce output causal graphs that are not Markov to the population distribution. 

A remedy is to find ways to verify the specific claims implied by the learned DAG model. Two assumptions that are necessary for the correctness of a wide variety of causal structure learning algorithms are the Causal Markov Assumption (CMA) and a simplicity assumption. The Causal Markov Assumption states that given a causal graph that is a directed acyclic graph and that contains a set of variables for which there are no unmeasured confounders, each vertex is independent of its nondescendants conditional on its parents. It has several equivalent formulations, including that if a set of variables \textbf{X} is d-separated from a set of variables \textbf{Y}, conditional on a set of variables \textbf{Z} (where \textbf{X}, \textbf{Y}, and \textbf{Z}) are disjoint) then \textbf{X} is independent of \textbf{Y} conditional on \textbf{Z} in the population distribution. There are several different simplicity assumptions commonly made, but one that is recommended by Raskutti and Uhler is the edge minimality assumption, that is, no DAG that is Markov to the population distribution has fewer edges than the true causal DAG \cite{raskutti2018learning}. The Markov Checker does a statistical test of whether d-separations in the learned causal model are conditionally independent in the population distribution by getting p-values for statistical tests of conditional independence entailed by the output model and then checking whether the p-values are uniformly distributed, as theory requires to be the case under the null assumption that the output model satisfies the Causal Markov Assumption. CAFS obtains search results from various algorithms under various search parameter settings, and eliminating ones that fail the Markov Checker, and further eliminating ones that don't have the minimal number of edges among the models that pass the Markov Checker. 

This paper proceeds as follows. First, we review the Markov condition, after which we describe Raskutti and Uhler's recommendations and show how they can be repurposed beyond the specific algorithm they propose, to compare the output of different algorithms and different search parameters. We then describe our statistical test of whether the Markov condition holds for a given DAG or CPDAG and the population distribution, and discuss its implementation. We follow this with simulation and real data results showing how one can use the Markov Checker to aid in selecting good models without knowing the ground truth and the limitations of the Markov Checker. We discuss how this expands the toolkit available to domain experts for analyzing real data. Finally, we give some conclusions and prospects.

Our novel contributions are as follows.
\begin{itemize}
    \item We provide the Markov Checker, a nonparametric pointwise consistent test for DAG or CPDAG models of whether or not the Markov condition holds.
    \item We provide a software tool, the Markov Checker, which helps, using the CAFS algorithm, to identify simplest models passing a Markov condition, usable on multiple platforms.
    \item We show how this software can be used to select accurate models from a list of algorithm outputs containing an accurate model and can be used to tune parameter settings, without knowing the ground truth.
     \item We describe limitations of the Markov Checker and CAFS.
    
\end{itemize}

\section{Preliminaries}

We define some terms that we will need. A \textit{graph} $G$ is an ordered pair $\langle E, V \rangle$, where \textit{V} is a set of \textit{vertices}, and \textit{E} is a set of \textit{edges} over $V$, and each edge $\langle x, y \rangle$ is either \textit{directed} from $x$ to $y$ with a \textit{tail} at $x$ and and \textit{arrow} at $y$ (we represent these as $x \rightarrow y$) or \textit{undirected} with a tail at both $x$ and $y$ (we represent these as $x - y$. We denote vertices with lowercase letters, and we denote sets of variables with uppercase letters. We refer to the arrows and tails in a graph as \textit{endpoint markings}. An \textit{adjacency} is an edge in a graph regardless of endpoint markings. A path is a list $\langle x_1, x_2, \dots, x_n \rangle$ in $G$ where $\langle x_i, x_{i+1} \rangle$ is an edge in $G$ for each $i$. A \textit{directed path} is a path where $\langle x_i, x_{i+1} \rangle$ is a directed edge in $G$ from $x_i$ to $x_{i + 1}$ for each $i$. A \textit{cycle} is a directed path from a node to itself.  A \textit{directed graph} is a graph with only directed edges. A \textit{directed acyclic graph} (DAG) is a directed graph without cycles. A \textit{collider} is a path of the form $x \rightarrow y \leftarrow z$. A \textit{noncollider} is a path $\langle x, y, z \rangle$ that is not a collider. A \textit{v-structure} is a collider $x \rightarrow y \leftarrow z$ where there is no adjacency between $x$ and $z$ in the graph; if there is an adjacency between \textit{x} and \textit{y}, we say that the collider is \textit{covered}. A node $a$ is an ancestor of $b$ if there is a direct path from $a$ to $b$ or $a$ = $b$; we also say that $b$ is a descendent of $a$ in this case.

A path from $x$ to $y$ in a DAG is \textit{blocked} given a set \textit{Z} of variables if it contains a collider with no descendant in \textit{Z} and is \textit{unblocked} otherwise. We say that \textit{X} and \textit{Y} are \textit{d-connected} just in case some path from a member of \textit{X} to a member of \textit{Y} is unblocked and \textit{d-separated} otherwise (which we notate as $dsep(x, y \mid Z)$.

We say that $X$ and $Y$ are \textit{conditionally independent} in $P$ (denoted $I(x, y \mid Z)$ given $Z$ iff $Pr(X \mid Y, Z) = Pr(X \mid Z)$ in $P$. Note that  if \textit{X} is independent of \textit{Y} conditional on \textit{Z} that it follows that each \textit{x} in \textit{X} is independent of each \textit{y} in \textit{Y} conditional on \textit{Z}. However, the converse is \textit{not} true in all cases, although it is true for Gaussian distributions. 

We may associate with a set of vertices that are random variables in a graph $G$ a \textit{distribution} $P$; We say that $G$ is \textit{Markov} for $P$ just in case $dsep(x, y \mid Z)$ implies that $I(x, y \mid Z)$ in $P$ and \textit{faithful} to $P$ just in case $I(x, y \mid Z)$ in $P$ implies that $dsep(x, y \mid Z)$ in $G$. In this way, a DAG $G$ implies, under the Markov relation, a set of conditional independencies. DAGs that imply the same set of conditional independencies under the Markov relation (i.e., have the same d-separation relations) are in an equivalence class we call a \textit{Markov equivalence class}. All DAGs in a Markov equivalence class can be represented by a \textit{completed partially directed acyclic graph} (CPDAG) which are graphs that can contain both directed and undirected edges. A CPDAG has the same adjacencies as each DAG in the Markov equivalence class it represents, and a directed edge from \textit{x} to \textit{y} if every DAG in the Markov equivalence class has an edge from \textit{x} to \textit{y}, and an undirected edge between \textit{x} and \textit{y} otherwise; see \cite{spirtes2000causation}.

We may draw a sample dataset $D$ from $P$ by selecting units independently that are distributed as $P$. Variables that are in a dataset we call \textit{measured}; variables not in a dataset we call \textit{latent}. If a latent variable $l$ is such that $x \leftarrow l \rightarrow y$ for variables in our data, we call $l$ a \textit{latent common cause}. We say that the assumption of \textit{ causal sufficiency} is satisfied for a data set $D$ just in the case that there are no latent common causes for $D$; otherwise we say that \textit{causal insufficiency} holds. For the purposes of this paper, we will assume that data sets are causally sufficient, though we will make some remarks about analyses that might be performed under causal insufficiency.

We assume that we have available pointwise consistent tests of conditional independence for the class of distributions under consideration; we will restrict ourselves for purposes of exposition to conditional correlation for linear, Gaussian models, but there are also nonparametric tests of conditional independence (e.g., \cite{zhang2012kernelbasedconditionalindependencetest}). Such tests generate \textit{p-values} in the usual way. For these tests, we can specify a \textit{significance threshold}--that is, a lower bound on p-values that are judged as conditionally independent--though such a threshold will not be used for our purposes.

\section{The Markov Condition}

We may further remark on the Markov relationship mentioned above. Usually, one makes the Causal Markov Assumption in algorithm construction for a true unknown causal graph. If the Causal Markov Assumption is true, and the output of causal search algorithm is correct, then the output CPDAG will be Markov to the population distribution. However, the Markov condition (MC) itself \citep{spirtes2000causation} is an assumption that can be checked for any DAG. For DAG models, we can state the Markov condition in terms of a relationship between d-separation claims for a DAG (or CPDAG) $G$ and independence facts that hold in a probability distribution $P$ over the variables in $G$

\begin{definition}[Global MC]
$(G, P)$ satisfies the Markov condition if, for sets $X, Y$ and $Z$, $dsep(X, Y \mid Z)$ in $G$ implies that $I(X, Y \mid Z)$ in $P$.
\end{definition}

Since every DAG represented by a CPDAG has the same set of d-separation relations, by extension we can also say that if \textit{G} is a CPDAG, then $(G, P)$ satisfies the Markov condition if, for for every DAG \textit{H} represented by \textit{G}, sets $X, Y$ and $Z$, $dsep(X, Y \mid Z)$ in $H$ implies that $I(X, Y \mid Z)$ in $P$.

In contrast, the Causal Markov Condition (CMC) is as follows. 

\begin{definition}[CMC]
$(G*, P)$ satisfies the Causal Markov condition if (i) $G^*$ is the true causal DAG for $P$ and (ii) $(G^*, P)$ satisfies MC.
\end{definition}

A local version of MC can be given as follows \citep{spirtes2000causation}.

\begin{definition}[Local MC]
$(G, P)$ satisfies the local Markov condition if every variable in $G$ is independent of its non-descendants given its direct parents.
\end{definition}

There is an advantage in calculating local Markov with respect to the given order of the variables, where a \textit{valid order} of the variables is a permutation $\langle x_1, x_2, \dots, x_n \rangle$ of the variables where parents of each node in the CPDAG occur before that node in the permutation.

\begin{definition}[Ordered Local MC]
$(G, P)$ satisfies the Ordered Local Markov condition with respect to a valid order of the variables of $G$ if every variable $x$ in $G$ is independent of its nonparents that precede it in the valid order given its direct parents.
\end{definition}

In the following, ``MC'' will refer to Ordered Local MC. The converse assumption to CMC is the Causal Faithfulness Condition (CFC), which may be put this way:

\begin{definition}[CFC]
$(G^*, P)$ satisfies the Causal Faithfulness condition if (i) $G^*$ is the true causal graph for $P$ and (ii) for each $X, Y$ and $Z$, $\neg dsep(X, Y \mid Z)$ implies $\neg I(X, Y \mid Z)$.
\end{definition}

The CMC is fairly well established, although there are some cases where it may fail \citep{zhang2008detection}. Violations of CFC may reflect path cancellation and may not indicate errors in the causal model itself\footnote{See \cite{lam2022greedy} for an oracle study.}. Under fairly weak assumptions, this can happen only with sets of parameter values of Lebesgue measure zero \citep{spirtes2000causation}, though ``almost violation of faithfulness'' can occur with positive probability
(\citep{uhler2013geometry}).

If CMC is true and MC does not hold for a graphical model $G$ given a data distribution $P$, then $G$ cannot be in $M(G^*)$. 

Note that if v-structures in the true model are mistakenly rendered as noncolliders in the estimated model, then they must be covered in the estimated model for the estimated model to satisfy Markov, as suggested by this lemma from \cite{chickering2002optimal}:

\begin{theorem} [Chickering Lemma 28]
Let $G$ and $H$ be two DAGs such that the set of conditional independences implied by $G$ is a subset of the set of conditional independences implied by H, that is, $G$ is an \textit{I-map} of $H$. If $G$ contains the v-structure $X \rightarrow Z \leftarrow Y$, then either H contains the same v-structure or $X$ and $Y$ are adjacent in $H$.
\end{theorem}

This theorem motivates Raskutti and Uhler's idea that at least for faithful models, the models one wants to choose are Markov models with a minimal number of edges, that is, Markov models with as few coverings of misspecified unshielded colliders as possible. Raskutti and Uhler also extend this motivation to many unfaithful models.

\section{Repurposing Raskutti and Uhler's Recommendations}

We next discuss how to repurpose Raskutti and Uhler's idea to compare DAG algorithms more generally.

\cite{raskutti2018learning} indirectly recommend checking two assumptions for DAG models when giving justification for their Sparsest Permutation (SP) algorithm. They first define the sparsest Markov representation (SMR) assumption as follows. Let $|G|$ be the number of edges in $G$ and let $M(G)$ be the Markov equivalence class (MEC) of $G$. 

\begin{definition}
\label{SMR}
A pair $(G, P)$ satisfies the SMR assumption if $(G, P)$ satisfies the Markov property and $|G'| > |G|$ for every DAG $G'$ such that $(G', P)$ satisfies the Markov property and $G' \not \in M(G)$.
\end{definition}

\noindent \cite{lam2022greedy, lam2023thesis} call Definition \ref{SMR} \textit{unique frugality (u-frugality)} and a model that satisfies u-frugality \textit{u-frugal}. A key feature of Raskutti and Uhler's Sparsest Permutation (SP) algorithm is that it searches for u-frugal DAG models. In cases where such a model exists, they prove the SP algorithm is correct.

\begin{theorem}[Raskutti and Uhler, Theorem 2.3]
\label{raskutti_uhler_theorem}
Let $G^*$ be the true model for distribution $P$. SP algorithm outputs $G_{SP} \in M(G^*)$ given distribution $P$ if and only if the pair $(G^*, P)$ satisfies the SMR assumption. 
\end{theorem}

\noindent Importantly, as \cite{raskutti2018learning} and \cite{lam2022greedy} point out, frugal models do not necessarily belong to the same MEC; there may be multiple MECs $\{M_i\}$ (including $M(G^*)$) with minimal edge count that contains the distribution. \cite{raskutti2018learning} gives a simple example where two MECs explain a particular independence model equally well with minimal edge count. The unique frugality assumption can be relaxed to accommodate this fact; \cite{lam2022greedy} call the relaxation \textit{frugality}. The SP algorithm will return all frugal models, but is computationally infeasible for large numbers of variables. If unique frugality is not satisfied, these may not all belong to a single MEC.\footnote{This result requires a regularity condition, such as assuming the positive distribution.} We may therefore relax Theorem \ref{raskutti_uhler_theorem} as follows:

\begin{corollary}[Raskutti and Uhler Relaxation]
\label{frugal_model_class}
Let $G^*$ be the true model for the positive distribution P. Then SP outputs $G$ in some minimum edge count MEC $M \in \{M_i\}$ (containing $M(G^*)$) given a distribution $P$ if and only if $(G, P)$ satisfies the assumption of frugality.
\end{corollary}

Corollary \ref{frugal_model_class} gives a condition (frugality) that can be checked in the data, or at least \textit{rejected} for empirical DAG models: check which models satisfy the Markov condition and of the models that do satisfy the Markov condition, reject those that do not have minimal edge count. We can reject models that are not Markov as not explaining the conditional dependence facts implied by the data. Models that are Markov but with more than the minimal number of edges we can reject as being non-frugal, not as simple as possible, and thus violating the condition of Corollary \ref{frugal_model_class}. Raskutti and Uhler project variable permutations to DAGs as recommended using a particular conditional independence test. In this way, we can sort through DAG models and pick out ones that cannot be rejected out of hand by this frugality criterion. This procedure mimics how the SP algorithm works, though SP considers all permutations of the variables and therefore cannot be applied to many more than a dozen variables. We may reject DAG models using the same reasoning. Still, we can check the Markov condition for much larger models, provided that we have algorithms that produce examples of such DAG models or Markov equivalence classes thereof.

\section{A Test of the Markov Condition}

Next, we describe our proposed check of the Markov condition for an empirical graphical model given a data set.

We can check Markov by making a list of graphical conditional independence claimed to hold made by an empirical graph $G$ using the d-separation criterion $S$ for the model class it represents and then check the correctness of this list. Checking to ensure that all independence facts implied by an empirical model are judged as independent by a statistical test of independence is not quite what is needed. Statistical tests of conditional independence use a significance cut-off point ($\alpha$) to reject independence, assuming that the null hypothesis of independence is true. If the null hypothesis is true, false rejections of the null hypothesis will occur at a rate of $\alpha$, which are unavoidable errors. False dependence judgments can also be rendered for similar reasons because of the distribution of p-values under the alternative dependence hypothesis. So, there will inevitably be some overlap between the judgments for the lists of implied independence and dependence judgments if the test yields the correct number of Type I errors even if the graph in question is $G^*$, provided enough independencies are checked.

We may check the Markov condition statistically by testing the uniformity of p-values under the null for independence facts implied by an empirical graphical model. Under the null hypothesis, the distribution of p-values with a continuous CDF is distributed as $U(0, 1)$, which implies the following theorem.

\begin{theorem}[Uniformity Check]
\label{causal_markov_check}
Let $G$ be a graphical model over nodes $V$. Let $P_{x, y, Z}$ be the population distribution of $\langle x, y \rangle$ conditional on $Z$, where $x$ and $y$ are nodes in $V$, $Z$ is a set of nodes in $V$, and assume that $P_{x, y, Z}$ in each case has a continuous CDF. Let $P$ be the list of p-values obtained by evaluating each of these conditional distributions on independent samples. Then $P$ is distributed as $U(0, 1)$.
\end{theorem}

\begin{proof}
This theorem follows immediately; any collection of independently obtained p-values for variables with continuous CDFs constitutes independent draws from $U(0, 1)$ by \cite{casella2021statistical} Theorem 2.3 (``probability integral transformation'') and will therefore be distributed as $U(0, 1)$.
\end{proof}

\noindent By Theorem \ref{causal_markov_check}, if one correctly estimates a p-value for each implied independence or conditional independence in a graphical model, and if these p-values are obtained from independent samples, are sufficiently many to do a test and are found not to be distributed as $U(0, 1)$, one can reject that MC holds for this model---that is, that the model contains the distribution of the data. The only question is how many p-values one needs to do this test in practice to have sufficient Power to reject alternative models; we can check this in simulation. We will comment on dependence among p-values shortly. Note that such a Uniformity check has the advantage that non-uniformity can be detected in principle throughout the distribution of p-values. Our algorithm for checking Markov is given in Algorithm \ref{alg:markov-check}.

\begin{algorithm}[ht]
\caption{MarkovCheck}
\label{alg:markov-check}
\KwIn{$G$, $F$, $D$, $I$, $U(p\_list)$, $\alpha$}
\KwOut{A judgment of whether $G$ passes a Markov check relative to $D$}
\KwData{$G$ is a graph, $F$ is a list of separation facts implied by $G$, $D$ is a dataset over the variables in $G$, $I$ is a test of independence, $U$ is a test that returns the p-value of a judgment of whether $p\_list$ is drawn from $U(0, 1)$, and $\alpha$ is a significance cutoff for $U$}

Initialize an empty list of p-values, $p\_list$\;
\For{each fact $f$ in $F$}{
    Use the independence test $I$ to compute a p-value for $f$\;
    Append the p-value to the list $p\_list$\;
}
Compute the p-value of $p\_list$ using $U(p\_list)$\;
\If{$U(P) > \alpha$}{
    \Return True\;
}
\Else{
    \Return False\;
}
\end{algorithm}

In this implementation of the Markov Checker, the algorithm assumes that the p-values are estimated for \textit{x} independent of \textit{y} conditional on \textit{Z}. As noted above, it is not the case that conditional on \textit{Z} independence of each \textit{x} in \textit{X} and \textit{y} in \textit{Y} entails in general that conditional on \textit{Z} that \textit{X} is independent of \textit{Y}. Thus, it is possible for some families of distribution that a DAG that does not satisfy the Markov condition for the population distribution would not be rejected by the Markov Checker. However, this will not occur in the simulations that we have performed in this paper, since in the Gaussian distribution that we use in the simulations, dependence of X and Y conditional on Z does entail dependence of some x and y conditional on Z.

Strictly speaking, if the Markov Checker rejects a hypothesis, it is not specifically rejecting that the DAG is Markov to the population distribution but instead rejecting that the DAG is Markov to the population distribution \textit{and} that the p-values obtained under the null are independent. If we were to look at the histogram of p-values for the independence relations implied by $G$ under null, the dependence of the test would appear as a spike in the leftmost bar; the rest of the bars should be uniformly distributed under either independence or dependence if the p-values are independent. If the p-values are not uniform for the remainder of the p-values, one knows that the p-values are themselves dependent.

In any case, there are things that we can do in principle to ensure that our various tests of independence are at least reasonably independent of each other so that we can apply the tests sensibly. We could, for instance, perform tests on independent samples, by data splitting, say, though this is not always feasible with small samples. We could also ensure the independence of samples by subsetting the independence tests we do so that the variables they test do not overlap, though, this may result in a set of independence tests that is no longer representative of $G$, if the model requires large conditioning sets. If small conditioning sets can be found, this may be advantageous; one may for this reason wish to avoid checking all conditional independence tests implied by a model and look only at Local Markov or Ordered Local Markov.

We may also, in principle, encourage the independence of p-values by minimizing data overlap for conditional independence tests through subsampling for each conditional independence test performed. To minimize the overlap of variables by subsampling, we added a parameter to our software to pick a random subsample without replacement for each new conditional independence test of a certain percentage of the available samples. \cite{meinshausen2010stability} and \cite{stekhoven2012causal} also recommend this strategy and suggest that this proportion be set by default to 0.5; we take this as our default but allow the user to choose the size of the sub-sample. 

In practice, adjusting for these concerns is unnecessary, as we demonstrate in the next section.

\section{Data Overlap}\label{sec:overlap}

One worry about collecting p-values from multiple conditional independence tests is that tests with variables in common may lead to dependent p-values, thereby violating uniformity under the null. A possible solution to this problem, as noted, is to limit the overlap of data from one test to the next. In simulations, we found that this is unnecessary; see Figure \ref{fig:data-overlap}. The figure shows the proportion of Anderson-Darling p-values less than a given value plotted against the estimated Anderson-Darling p-values; if these p-values are uniformly distributed, one expects a diagonal line from (0, 0) to (1, 1) in the plot. We find that all three conditions approximate this line well. So, despite the theoretical worry, we do not find it helpful in practice to insist on data subsetting as a way of generating more uniformly distributed sets of p-values for conditional independence tests. 

In particular, the figure shows results for a model with 25 nodes and an average degree of 5 (to match the types of models analyzed in our simulation section below). These are for linear Gaussian models, using the DaO simulation method \citep{andrews2024better}. The figure on the left uses a sample size $N = 100$; the figure on the right uses $N = 1000$. The number of repetitions is 100. Three cases are shown. In the ``sub'' case, every test uses a different $N/2$-size subsample of the data. In the ``orig'' case, every test uses the full original data set. In the ``new'' case, every test draws a new full dataset. In all cases, the Kolmogorov-Smirnoff tests did not find any significant differences; all pairwise comparisons have p-values $> 0.05$; see Table \ref{overlap_p_values}. It is interesting that even for $N = 100$ one observed distributions of p-values very close to $U(0, 1)$ in all three cases, showing that uniformity under the null is not an issue at that sample size (although detecting dependencies may be).

\begin{table}[h!]
\centering
\begin{tabular}{l|c|c}
& N = 100 & N = 1000 \\
\hline
sub-orig & 0.4695 & 0.5830 \\
sub-new  & 0.4695 & 0.3682 \\
orig-new & 0.7021 & 0.5830
\end{tabular}
\caption{P-values for different comparisons at N=100 and N=1000}
\label{overlap_p_values}
\end{table}

However, the power of the Markov Checker test depends upon several factors: the independence, number, power and correctness of the statistical tests producing the p-values input into the uniformity test. The power of the statistical tests producing the p-values in turn depends on the sample size of the data. For a given dependence between \textit{x} and \textit{y} conditional on \textit{Z}, the p-values grow smaller as the sample size grows larger. This leads to increased non-uniformity in the p-values as the sample size grows larger, and the increased non-uniformity is easier to detect by the uniformity test as the sample size grows larger. 

In terms of power, a procedure $P_{n,m}$ that partitions the data in such a way that each single conditional independence is tested by \textit{n} independent samples of sample size $m$ has both an advantage and a disadvantage over a procedure $P_{1,n \times m}$ in which each single conditional independence is tested once by a sample that of size $n \times m$. On the one hand, the power of the uniformity test is increased for $P_{n,m}$, because more p-values are provided to the uniformity test. On the other hand, because each p-value is generated by a smaller sample size in $P_{n,m}$, the non-uniformity is more difficult to detect. Exactly how the trade-off between $P_{n,m}$ and $P_{1,n \times m}$ works for a given graph and sample is not known.

\begin{figure*}
    \centering
    \begin{subfigure}[t]{0.45\textwidth}
    \centering
    \includegraphics[width=\textwidth]{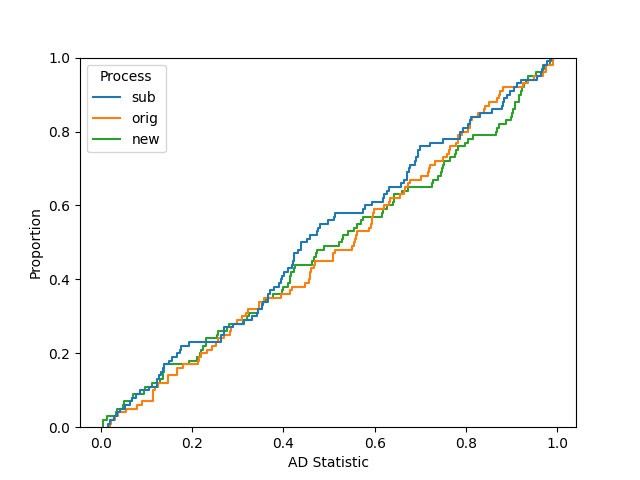}
    \caption{100 samples}
    \end{subfigure}
    \hfill
    \begin{subfigure}[t]{0.45\textwidth}
    \centering
    \includegraphics[width=\textwidth]{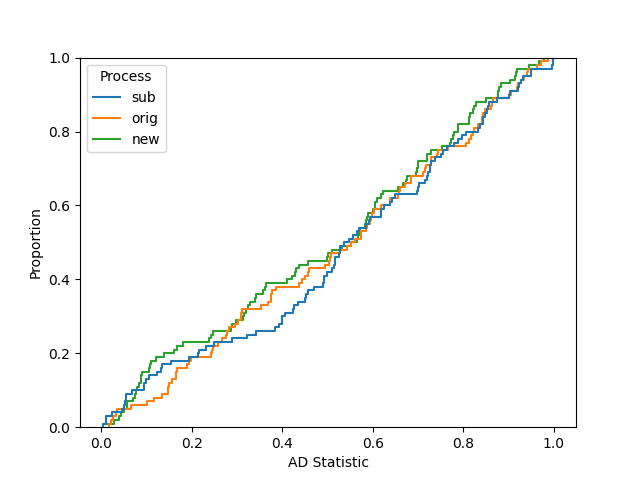}
    \caption{1000 samples}
    \end{subfigure}
    \caption{A comparison of p-value uniformity for conditional independence testing, for (i) the dataset randomly subsetted without replacement to sample size $n / 2$, (ii) the original dataset, without subsetting, and (iii) a newly sampled dataset.}
    \label{fig:data-overlap}
\end{figure*}

\section{The CAFS Procedure}

Given a data set $D$, we can test various candidate graphs under Markov conditions and compare uniform p-values and the number of edges in these graphs. This lets us define a metaloop over a set of candidate DAGs or CPDAGs. Candidate DAGs or CPDAGs can come from background beliefs, algorithms with fixed search parameters, or some combination of the two. CAFS will then search for graphs that cannot be rejected by either a Markov check or the frugality criterion. We call this meta-loop ``Cross-Algorithm Frugality Search'' (CAFS) and give the algorithm in Algorithm \ref{alg:cafs}. Importantly, it does not matter which principles or background assumptions that the various algorithms use to search the space of DAG models; in our experimental section below, we consider algorithms applied to simulation data that appeal to very different types of theory and compare them directly to one another using the notion of frugality. Importantly, CAFS does not require any reference to ground truth, since Causal Markov is a relationship between $P$ and $G$, and edge counting edges depend only on $G$. We hypothesize that CAFS helps to pick out more accurate models from a list of graph options, and perform a simulation study below to confirm this.

It is not necessary to use the number of edges as the simplicity criterion for such a search; other simplicity criteria, such as the parameter count, could be substituted. Further work could explore the use of various simplicity criteria for this purpose.

\begin{algorithm}[ht]
\caption{Cross-Algorithm Frugality Search (CAFS)}
\label{alg:cafs}
\KwIn{$A$, $D$, $I$, $U$, $\alpha$}
\KwOut{The algorithm in $A$ with the fewest number of edges that passes the Markov check}
\KwData{$A$ is a list of candidate DAGs or CPDAGs, $D$ is a dataset, $I$ is an independence test, $U$ is a uniformity check, $\alpha$ is a significance cutoff for the Markov check}

Initialize $best\_alg \gets \texttt{None}$\;
Initialize $min\_edges \gets \infty$\;

\For{each candidate graph $G$ in $A$}{
    Infer a list of conditional independence facts $F$ implied by $G$\;
    Record the number of edges in $G$ as $e$\;
    
    Run \texttt{MarkovCheck($G$, $F$, $D$, $I$, $U$, $\alpha$)}\;
    
    \If{\texttt{MarkovCheck} returns True \textbf{and} $e < min\_edges$}{
        $best\_alg \gets a$\;
        $min\_edges \gets e$\;
    }
}

\Return $best\_alg$\;
\end{algorithm}

\section{Implementation}

We provide a software tool, the Markov Checker, for checking Markov as part of the Tetrad software suite for those who do not wish to handle the details themselves, though the test can be readily implemented. In the tool, one first selects an independence test for comparison from a list of available tests. d-separation is used to infer graphical predictions of independence for an empirical search graph.\footnote{Both m-separation and d-separation use the same algorithm; we call this m-separation in our software to highlight that our Markov check can also be applied to latent variables models, a topic we hope to return to in future work.} One then selects a proportion of the available sample size for each independence performed; the default is 0.5, as suggested above. By default, we test Ordered Local Markov, though other types of conditioning sets are available. For Ordered Local Markov, the number of tests to check is quadratic in the number of variables. For this reason, our software easily scales to more than 200 variables using a linear Gaussian conditional independence test on a Mac laptop with an M2 processor, in a single thread, and the procedure is highly parallelizable.\footnote{We use an optimized check of d-separation using reachability \citep{geiger1990} where, for example, ancestors in $G$ are precomputed, and parents in $G$ are hashed. We also restrict ourselves to fast independence tests and allow the procedure to be parallelized across different algorithm runs, a parallelization that is also highly parallel.}

Once selections are made, all implied m-separations are listed, along with their p-values and judgments of independence according to the chosen statistical conditional independence test. A histogram of these p-values is shown so that the user can visually judge their uniformity, although a formal uniformity test is more reliable and informative. For our Markov check, we give a p-value for an Anderson-Darling test to determine whether the claim that the test p-values are drawn from $U(0, 1)$ can be rejected \citep{anderson1952asymptotic}. For the p-value of the Anderson-Darling test, we use the formulation given in \cite{marsaglia2004evaluating}. We also give a Kolmogorov-Smirnoff statistic for the same proposition. Power has been studied for Anderson-Darling and Kolmogorov-Smirnoff Gaussianity tests by \cite{razali2011power}; we have not found a similar published study of Power for such tests of U(0, 1), although we can comment on the Power of our Markov tests based on our simulation result. We also give the edge count of the tested model so that the user may select models with low edge counts.

We provide a general grid search facility in our software to help users identify minimal edge models satisfying Markov, though again such a grid search facility is straightforward (perhaps tedious) to implement for particular cases. For our simulation comparison below, we simply list the algorithm outputs together with a variety of statistics, some of which require knowledge of the ground truth and some of which do not, and allow these tables to be sorted, first by sorting algorithm variants that pass Markov to the top and then secondarily sorting by simplicity. The measure of simplicity that we provide is that that Raskutti and Uhler recommend: that is, the number of edges in the estimated CPDAG, $|G|$. As a surrogate for the Markov checking p-value, to make the plots easier to interpret, we calculate the Kullback-Leibler distance from $U(0, 1)$ of a histogram of the conditional independence p-values for a given algorithm variant. In our software tool, this is the histogram of p-values given in the interface.

\section{Expanding the Toolkit for Domain Experts}

In a strict sense, BIC \citep{schwarz1978estimating, haughton1988choice}, or, for that matter, any penalized likelihood statistic, does not allow one to reject false models. Unless the correct model is among the models being compared, one can never know how far one is in practice from optimizing the penalized likelihood statistic. BIC is often used without ground truth by optimizing this score to the extent possible, even though the problem is recognized (e.g. \cite{teyssier2005ordering}). A hope is that exact searches can get around this stricture, though exact searches do not guarantee that the correct model is among the compared models at finite sample sizes; also, they may be searching the wrong class of models.

For linear Gaussian models, the p-value of a chi-squared test is often appealed to without ground truth, though in practice, if even one or two edges are out of place in one's model, the p-value of the model drops numerically to zero, a known drawback (\cite{bentler1980significance}). A response to this has been the development of various \textit{model fit indices}, particularly in the social science literature \citep{bollen1993testing}, such as the Comparative Fit Index (CFI, \cite{bentler1990comparative}) or the Normed Fit Index (NFI, \cite{bentler1980significance}), the so-called ``Bentler-Bonett score.'' These are also methods of algorithm comparison that are available to the empirical researcher without reference to the ground truth, though they suffer from the same problem as BIC, in that they do not give ``thumbs-up'' or ``thumbs-down'' judgements of particular empirical models. However, there are recommended cut-off points for accepting a model on either of these criterion (\cite{hu1999cutoff}). It will be interesting to see what these cutoffs compare to the recommendations of the Markov checker. We will include these two model-fit indices in our simulation outputs.

In sum, a check of the Markov condition with an eye toward minimal edge counts can expand the toolkit for domain researchers by allowing these large and dense models to be easily checked without ground truth, in a way that gives a ``thumbs-up'' or ``thumbs-down'' judgment that is motivated theoretically. 

\section{Limitations}

Notably, the Markov checker is looking for minimal Markov models, which amounts, for the causally sufficient case, to looking for CPDAGs that explain the data with the fewest number of edges, though other simplicity principles such as minimizing the number of parameters of the model could also be used. What it \textit{cannot} do is help pick out a particular model from a Markov equivalence class (MEC). Several algorithms are available that can do this under assumptions that are stronger than linear Gaussian. For example, if the model can be placed into the class of linear \textit{non-Gaussian} models, then it is possible to orient all edges in the model \citep{shimizu2011directlingam}, though the Markov checker would not be able to say whether these additional orientations are correct; some other test would be needed for that, such as checking the independence of residuals. Similarly, one can find more orientations than the MEC allows if the model can be correctly placed into the class of nonlinear models with additive noise or heteroskedastic models, and the Markov checker would not be able to say whether the additional orientations afforded by algorithms taking advantage of these stronger assumptions are correct, even if the data one has at hand clearly falls into one of these categories. For such data, it may still be reasonable to rule out numerous models, but getting such fine details of orientation correct is beyond its capabilities.

The procedure outlined above is that we are using a linear Gaussian conditional independence test for our Markov checker. This can be problematic if the data are not linear Gaussian. If the data are multinomial or mixed discrete/Gaussian, conditional independence tests are available for these cases that are fairly fast (cf. Chi Square or Degenerate Gaussian, \citep{andrews2019learning}). For more general distributions (and many empirical data sets do not fit in these distribution families), general tests would be preferable, in principle, although these are usually at least $O(N^3)$, where $N$ is the sample size (see, e.g. \cite{zhang2012kernelbasedconditionalindependencetest}). The question then becomes whether subsampling with or without replacement allows for sufficiently accurate assessments of conditional independence, or whether it is possible that faster tests can still give useful information despite being too weak to give good judgments of conditional dependence. This requires further study.

In addition, it is possible that the Ordered Local Markov check does not generate enough p-values by itself to perform a reliable test of whether the p-values are drawn from $U(0, 1)$. A possible remedy for this is to draw multiple random subsamples of the data and do the conditional independence tests implied by Ordered Local Markov for each subsample; as Figure \ref{fig:data-overlap} shows, such subsampling still generally yields p-values in $U(0, 1)$ in aggregate, so it is a sensible practice. In our software, we have included a feature that allows the user to do this; new independent p-values can be added.\footnote{In the Tetrad interface, it the Markov Checker tool, one clicks the button labeled ``Add Sample.'' This functionality is available as well in the Tetrad library, so it is available in Python as well.}

Another limitation of the method as we have presented it is that we are only looking for DAG or CPDAG models. This is not a limitation of the Markov checker \textit{per se}; one can extend the scope of the Markov checker to check Markov for any model class that admits a separation criterion. Partial ancestral graph models, for example, can use \textit{ m-separation} as a separation criterion to generate conditional independence facts implied by an acyclic model that allows for latent common causes and selection bias. However, this requires further work.

\section{A Simulation Comparison}

We determine how well a search for frugal models, iterating over multiple algorithms and tuning parameter choices, without referring to ground truth, picks out models judged as accurate by subsequent comparison to $M(G^*)$. We focus on the scenario where one is judging whether a specific model should be judged as passing Markov. One could also raise the question of which algorithms may be \textit{expected} to pass Markov under which settings of their hyperparameters; we will leave discussion of this question to another time. One feature of the Markov checker is that it is entirely possible that \textit{all} compared models fail a Markov check, so that there are no recommended models. We take this to be a feature and not a bug. As noted, a number of statistics in common use do not have this property but instead return a model maximizing or minimizing a score. So, if the set of models compared does not contain a suitable model, an inadequate model may be returned. The Markov checker aims to find adequate models, if not good or excellent ones, in a theoretically motivated way.

We consider models from the following algorithms for the linear Gaussian case that generate CPDAGs: PC \citep{spirtes2000causation}, CPC \citep{ramsey2006adjacency}, FGES \citep{ramsey2017million}, GRaSP \citep{lam2022greedy}, BOSS \citep{andrews2023fast}, DAGMA \citep{bello2022dagma}, BiDAG \citep{suter2023bayesian}, MMHC \citep{tsamardinos2006max}, and PCHC \citep{tsagris2021new}. We consider a linear Gaussian data only for these comparisons, and so leave out any algorithm that requires stronger (or different) assumptions. We simulate both training and testing data, learn the graphs on the training data, and test the Markov condition on the testing data. As for the models in Figure \ref{overlap_p_values}, these are for linear Gaussian models, using the DaO simulation method \citep{andrews2024better}. Random graphs are selected with a given number of nodes and a given average degree of nodes by adding uniformly selected random edges to the graph until the required number of edges has been achieved and directing the edges using a fixed order of the nodes. The number of nodes ranges over 5, 10, 15, 20, 25, and 30; the average degrees average over 1, 2, 3, 4, and 5. The size of the samples for the training and testing sets is 500, although as suggested in Figure \ref{overlap_p_values}, for each independence test that the Markov checker uses to generate a p-value, a random subsample of $N / 2$ is selected without replacement from the training set, in order to generate additional p-values where needed.

Notably, these algorithms use a variety of principles to infer CPDAGs. PC and CPC appeal to a conditional independence Oracle. In practice, they are not guaranteed to generate CPDAGs, so we need to check whether the resulting models are legal CPDAGs to use Ordered Local Markov and apply Corollary \ref{frugal_model_class}. We used a linear Gaussian conditional correlation test with $\alpha$ chosen from 0.001, 0.01, 0.05, 0.1, and 0.2. We include several score-based algorithms, FGES, GRaSP, and BOSS. These algorithms appeal to a linear Gaussian BIC score $BIC = 2 L - \lambda k (ln N)$ with penalty discount $\lambda$ chosen from 10.0, 5.0, 4.0, 3.5, 3.0, 2.5, 2.0, 1.75, 1.5, 1.25, and 1.0. GRaSP and BOSS follow the advice in \cite{raskutti2018learning} to search for models that seek suitable permutations of variables. We vary the range of $\lambda$ to demonstrate how our method recommends tuning parameter choices. DAGMA is a continuous optimization algorithm; we vary the parameter $\lambda_{1}$ as 0.01, 0.02, and 0.03, including the setting recommended by the authors $\lambda_{1}=0.02$. BiDAG uses Markov chain Monte Carlo methods to pick high-scoring DAGs; we use the authors' recommended settings. MMHC and PCHC are variant hybrid methods consisting of constraint-based steps to find adjacencies (different for each) and a score-based step to infer orientations. 

We apply CAFS using the Anderson-Darling test to determine whether the set of p values is drawn from a distribution $U(0, 1)$. We calculate a variety of statistics for the resulting models, some of which do not depend on knowledge of the ground truth, others do. In our full results, we consider problems for the numbers of variables in 5, 10, 15, 20, 25, and 30 and the average degrees of graph of graphs of 5, 10, 15, 20, 25, and 30, although for our present purposes we report only the case of 25 nodes with an average degree of 5. As argued above, it is possible to add additional p-values to check uniformity by performing tests with randomly subsampled data without replacement with sample size N / 2 (see Figure \ref{fig:data-overlap}). We have used this idea to make sure that each plot notes models passing our Markov check using a list of at least 200 p-values.

The full set of results and plots will be included on our GitHub site, together with Python scripts (and instructions) to run our simulations and plot arbitrary statistics. The Python scripts access code in Python, R, and Java.\footnote{Our GitHub site is at \url{https://github.com/cmu-phil/markov-checker}.}.

One statistic we report that does not require any knowledge of the ground truth, plotted in Figure \ref{fig:result_20_5_num_edges}, is the number of edges in the estimated CPDAG (``$|G|$'', \ref{fig:result_20_5_num_edges}), the CAFS result. We plot this against the p-value of the Anderson-Darling uniformity statistic. We plot with stars each point that passes a Markov check using Anderson-Darling with a significance level of 0.05, and with smaller round dots for other points. We plot the position of the true model in each plot in red and highlight the CAFS-selected model using an empty red star. What one expects from this graph, if the CAFS procedure is correct, is that a good model should be among those that are leftmost in the graph (low values for $kldiv$) and minimal among those (low values for $|G|$).  

One may wish to consider other starred points in the plot as alternative models passing a Markov check. To highlight one such point, we highlight with an empty circle the point picked out by a minimum KL divergence for the p-values obtained from the Markov checker against the p-values that would be expected under the assumption that they are obtained from a $U(0, 1)$ distribution. We do this as follows. We first calculate a histogram of the Markov checker p values with 20 bins. We then check this against a list of p-values equal to $1/20$ per bin. Other points passing Markov as well could be closely related to various simplicity conditions, but we will limiting ourselves to these two highlightings.

Another statistic that we plot that does not require knowledge of the ground truth is plotted in Figure \ref{fig:result_20_5_bic} BIC (Bayes Information Criterion, \cite{schwarz1978estimating}, \ref{fig:result_20_5_bic}), which we also plot against $p\_ad$. In our case, we calculate the BIC as $2L - k (ln N)$, where $L$ is the linear Gaussian likelihood of the model, $k$ the number of parameters and $N$ the sample size. Thus, we expect especially good models to have a high BIC score, calculated in this way, so we expect good models to be situated in the upper left-hand portion of the scatter plot.

As a way of seeing how these results compare to the truth, we plot two statistics also against $p_ad$ that require knowing the ground truth CPDAG for their calculation: (c) the F1 score for adjacency accuracy (calculated using adjacency precision and recall, where precision is defined as TP / (TP + FP) and recall is defined as TP / (TP + FN); see Figure \ref{fig:result_20_5_f1}), and (d) the Structural Hamming Distance (SHD) of the output CPDAG compared to the true CPDAG (counting as Hamming distance one error for each edge reversal and otherwise in one error each for any adjacency or arrowhead discrepancies, \cite{heckerman1995learning}; see Figure \ref{fig:result_20_5_shd}).

In addition, we include two model scores that do not require knowledge of ground truth for their calculation, (e) the CFI score of the estimated DAG (Comparative Fit Index, \cite{bentler1990comparative}; see Figure \ref{fig:result_20_5_cfi}) and (f) the NFI score of the estimated DAG (Normal Fit Index, \cite{bentler1980significance}; see Figure \ref{fig:result_20_5_nfi}).\footnote{We use the lavaan package in R (\cite{rosseel2012lavaan}) to calculate these model scores.}

We plot each of these additional statistics, respectively, against $p\_ad$, with plot icons as for $|G|$.

In our complete results, reported on our GitHub site, we plot a number of additional statistics against $p\_ad$.

\begin{figure*}
    \centering
    \begin{subfigure}[t]{0.49\textwidth}
    \centering
    \includegraphics[width=0.8\linewidth]{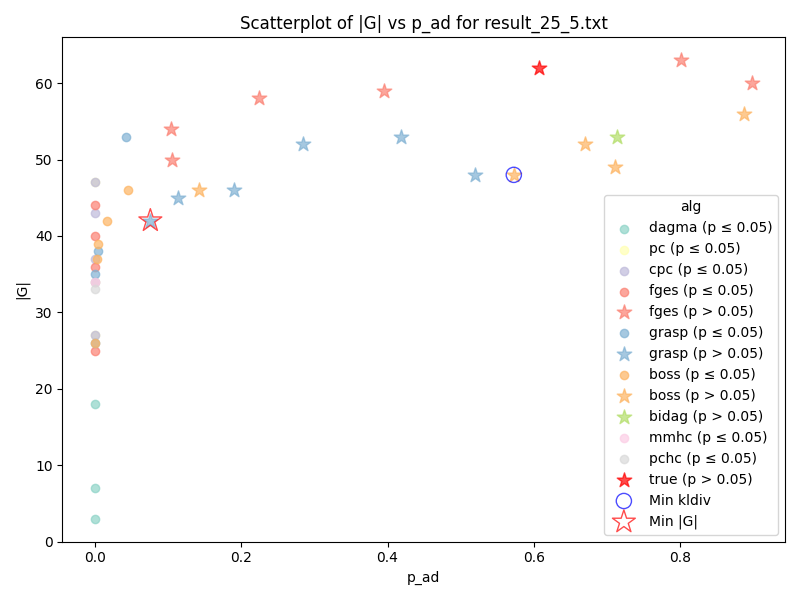}
    \caption{$|G|$: lower is better}
    \label{fig:result_20_5_num_edges}
    \end{subfigure}
    \hfill
    \begin{subfigure}[t]{0.49\textwidth}
    \centering
    \includegraphics[width=0.8\linewidth]{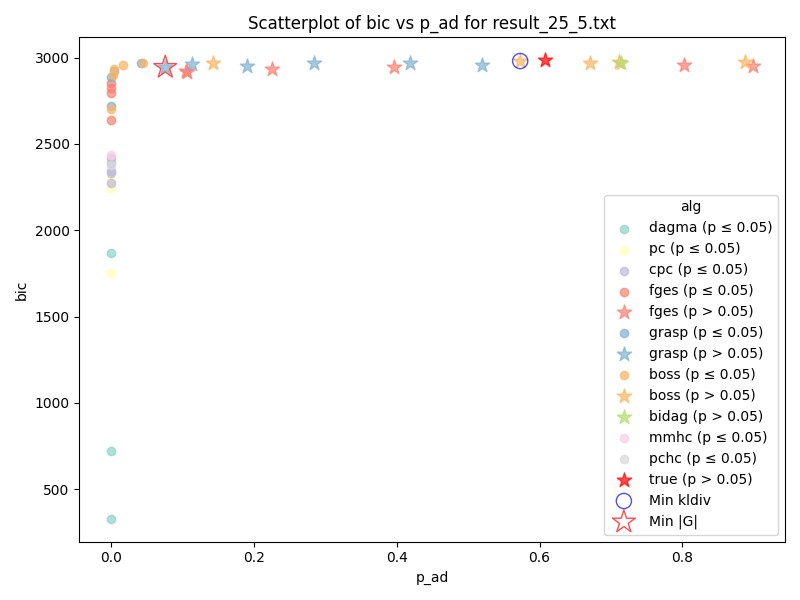}
    \caption{BIC: higher is better}
    \label{fig:result_20_5_bic}
    \end{subfigure}
    \begin{subfigure}[t]{0.49\textwidth}
    \centering
    \includegraphics[width=0.8\linewidth]{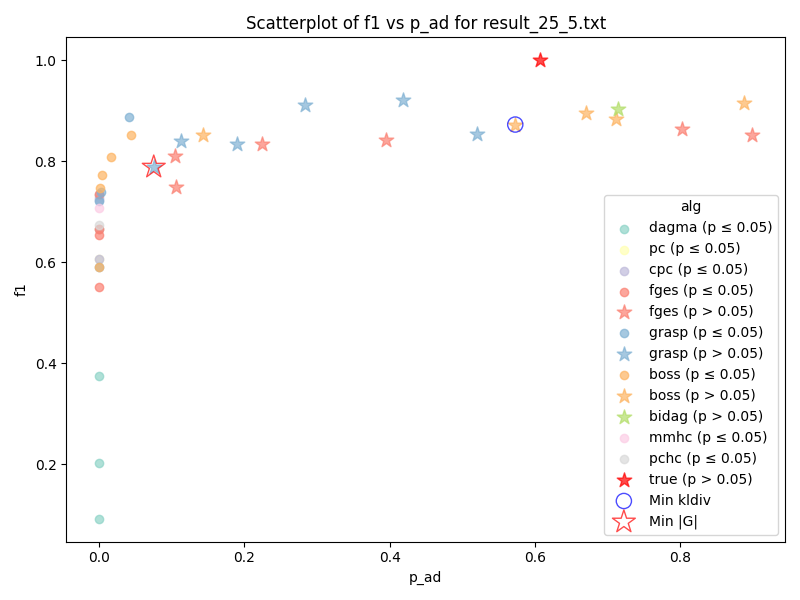}
    \caption{F1: higher is better}
    \label{fig:result_20_5_f1}
    \end{subfigure}
    \hfill
    \begin{subfigure}[t]{0.49\textwidth}
    \centering
    \includegraphics[width=0.8\linewidth]{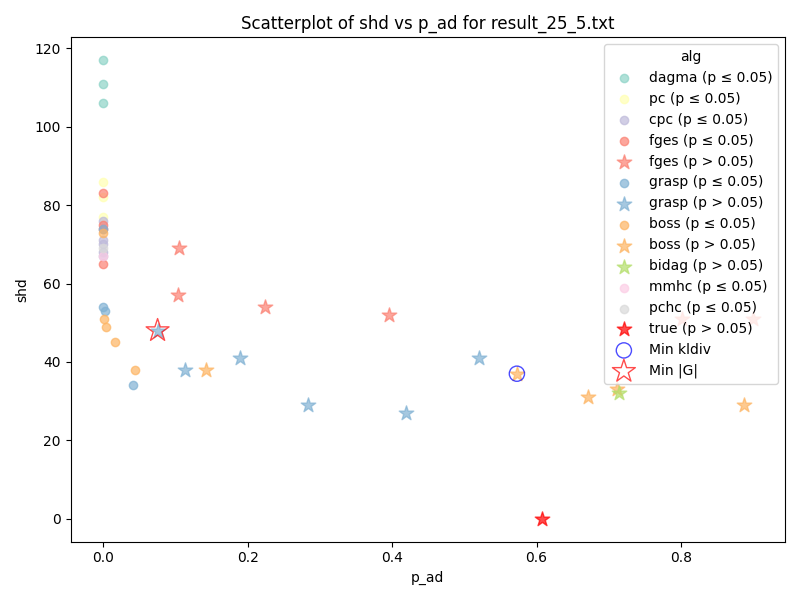}
    \caption{SHD: lower is better}
    \label{fig:result_20_5_shd}
    \end{subfigure}
    \begin{subfigure}[t]{0.49\textwidth}
    \centering
    \includegraphics[width=0.8\linewidth]{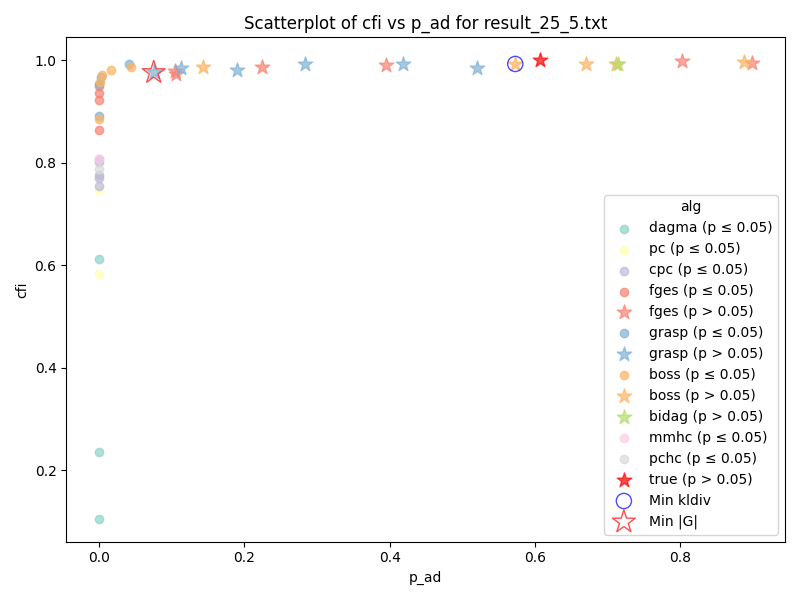}
    \caption{CFI: higher is better}
    \label{fig:result_20_5_cfi}
    \end{subfigure}
    \hfill
    \begin{subfigure}[t]{0.49\textwidth}
    \centering
    \includegraphics[width=0.8\linewidth]{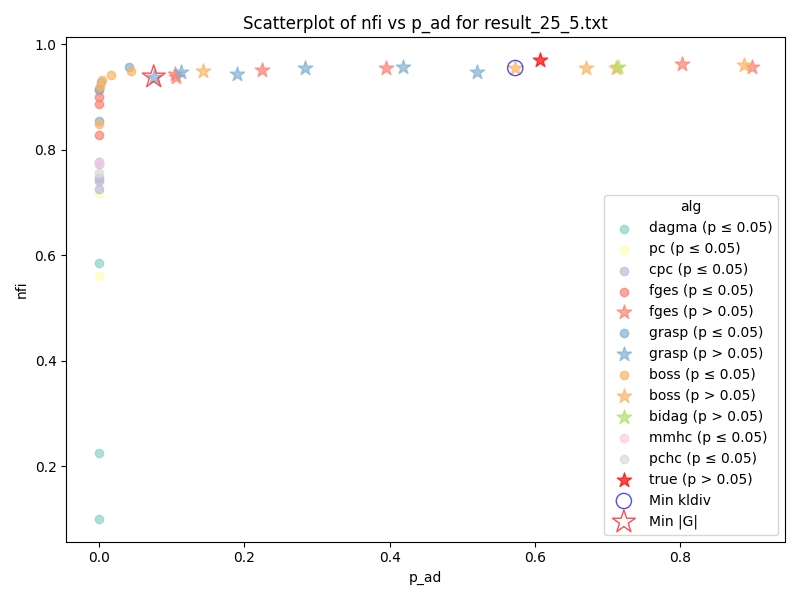}
    \caption{NFI: higher is better}
    \label{fig:result_20_5_nfi}
    \end{subfigure}
    \caption{Scatter plots of six statistics for the estimated CPDAGs against the Anderson-Darling p-value ($p_ad$) of a 20-bin histogram of the Markov checker p-values for each algorithm variant from expected p-values from a 20-bin $U(0, 1)$ histogram.}
    \label{fig:result_20_5}
\end{figure*}

\section{An Empirical Example}

As noted above, our presentation is subject to a number of limitations that make it somewhat difficult to find off-the-shelf real data sets easily to analyze. We choose a data set that has been analyzed many times, the US Crime data set.\footnote{This data may be found at \url{https://github.com/cmu-phil/example-causal-datasets/tree/main/real/uscrime}, together with a description of the variables and links to the source.} This has 14 variables with $N = 47$ and thus gives us the opportunity to consider the additional problem of small sample size. Since the sample size is so small, we will use an $\alpha$ threshold for our Anderson-Darling uniformity test of $0.2$. Since we do not have ground truth, we are not able to give plots for the F1 or SHD statistics, above, but we can give plots for $|G|$, BIC, CFI and NFI. Because they often do not produce legal CPDAGs for these data (in violation of our choice to use Ordered Local Markov to generate conditional independence facts implied by a model), we remove PC and CPC from our list of algorithms. Otherwise, we follow the same procedure as above.

Figure \ref{fig:us_crime_num_edges} shows the number of edges in the estimated model plotted against $p\_ad$ for these data; Figure \ref{fig:us_crime_bic} shows the BIC plotted against $p\_ad$; Figure \ref{fig:us_crime_cfi} shows CFI plotted against $p\_ad$; Figure \ref{fig:us_crime_nfi} shows NFI plotted against $p\_ad$.

\begin{figure*}
    \centering
    \begin{subfigure}[t]{0.49\textwidth}
    \centering
    \includegraphics[width=0.8\linewidth]{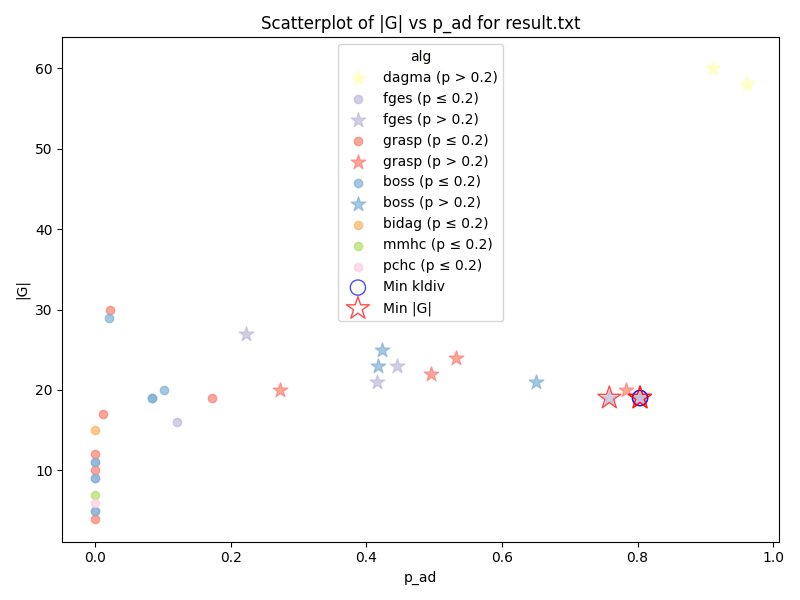}
    \caption{$|G|$: lower is better}
    \label{fig:us_crime_num_edges}
    \end{subfigure}
    \hfill
    \begin{subfigure}[t]{0.49\textwidth}
    \centering
    \includegraphics[width=0.8\linewidth]{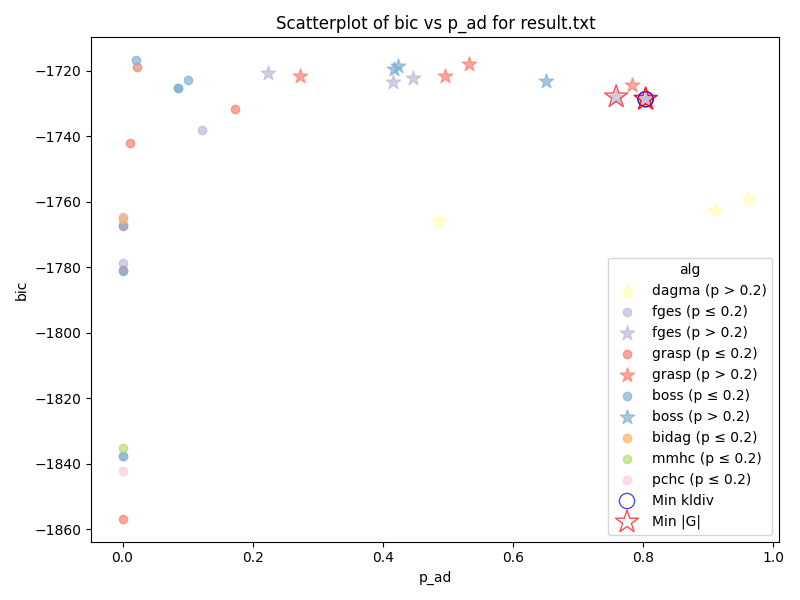}
    \caption{BIC: higher is better}
    \label{fig:us_crime_bic}
    \end{subfigure}
    \begin{subfigure}[t]{0.49\textwidth}
    \centering
    \includegraphics[width=0.8\linewidth]{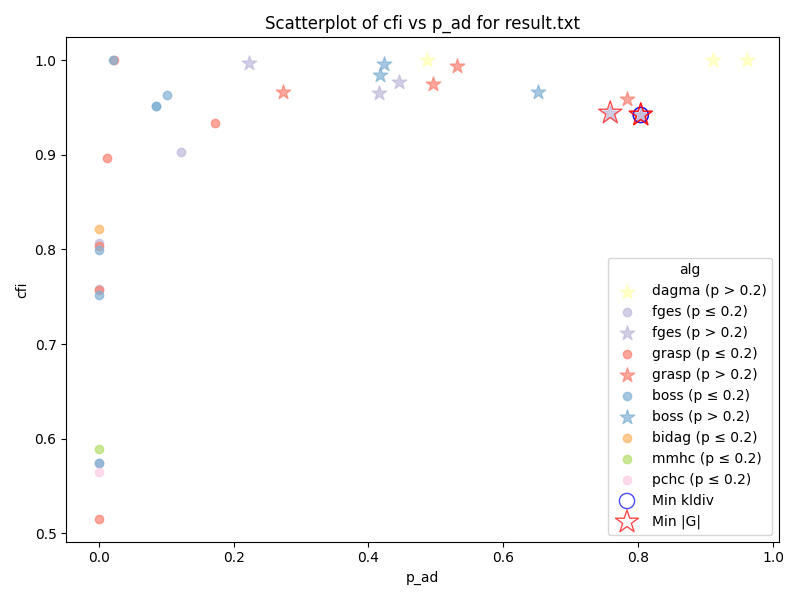}
    \caption{CFI: higher is better}
    \label{fig:us_crime_cfi}
    \end{subfigure}
    \hfill
    \begin{subfigure}[t]{0.49\textwidth}
    \centering
    \includegraphics[width=0.8\linewidth]{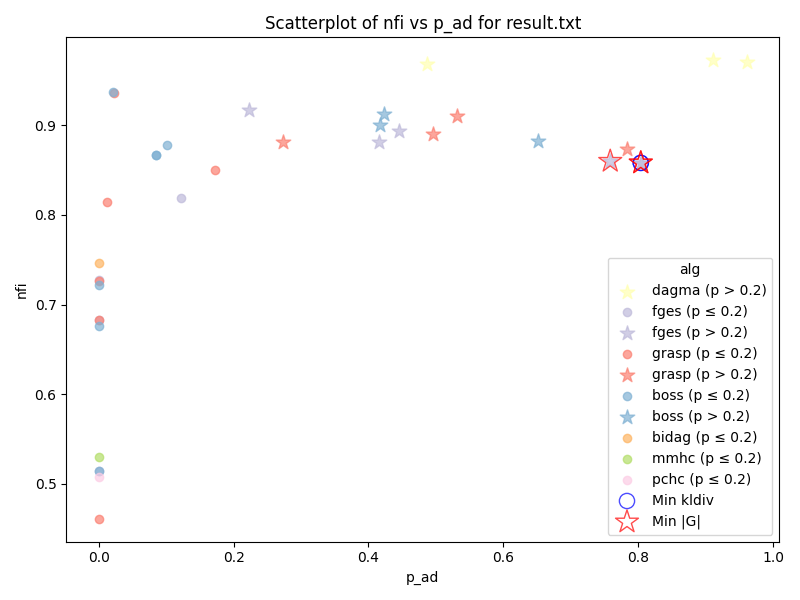}
    \caption{NFI: higher is better}
    \label{fig:us_crime_nfi}
    \end{subfigure}
    \caption{For the US Crime data, scatter plots of four statistics for the estimated CPDAGs against the Kullback-Leibler divergence ($kldiv$) of a 20-bin histogram of the Markov checker p-values for each algorithm variant from expected p-values from a 20-bin $U(0, 1)$ histogram.}
    \label{fig:us_crime}
\end{figure*}

\section{Discussion}

For the simulation data, regarding Figure \ref{fig:result_20_5_num_edges}, plotting the number of edges against $p\_ad$, the CAFS procedure suggests that we should consider the set of starred models and then choose some set of models with lower $|G|$. Looking at the scatter plot, the model with the least $|G|$ is a GRaSP variant. If we were to look at the full results, we could easily pick out which model this is, and if this were an empirical investigation, this would be a warrant to use one of these minimal or nearly minimal models for empirical investigation.

In Figure \ref{fig:result_20_5_bic}, we see that the models that the Markov checker selects are among those with the highest BIC scores. Note that BIC does not provide a threshold for rejection; in use for empirical models, one tries to maximize the BIC score, though the score itself does not specifically endorse or reject any model. Nevertheless, the Markov checker for this comparison picks out models with the highest score, suggesting that the Markov checker is able to pick out maximally explanatory models. The plot is readily divided into points along the y-axis and points that are horizontally arrayed across the top. The horizontally arranged points consist almost entirely of points that have significant Anderson-Darling uniformity p-values with $p > 0.05$. All are starred, including the CAFS point, the min $kldiv$ point, and the point for the true model.

Regarding Figure \ref{fig:result_20_5_f1}, looking just at the adjacencies of the model, we see that the Markov checker picks out models with among the highest F1 balance of precision and recall, with the F1 score of the CAFS model on the lower end of these, highlighting a recommendation discrepancy between adjacency F1 and the BIC score. Regarding Figure \ref{fig:result_20_5_shd}, disregarding models that are not legal CPDAGs, we see that the Markov checker picks out models with low structural Hamming distances. Again, the CAFS model does not have the lowest starred SHD score, suggesting that CAFS prefers to optimize BIC. Regarding Figures \ref{fig:result_20_5_cfi}, for the Comparative Fit Index (CFI), and Figure \ref{fig:result_20_5_nfi}, for the Bentler-Bonett Normed Fit Index, we see that CAFS picks a model with one of the highest scores in each case, very close to the score for the true model. Generally speaking, for both the CFI and the NFI, a value of $> 0.9$ is considered a good fit and a value $> 0.95$ is considered an excellent fit \citep{hu1999cutoff}. From the data, we see that the CAFS model has a CFI score of 0.9761 and an NFI score of about 0.9378, which are judged to have an excellent fit.

For an empirical example, a DAGMA variant in Figure \ref{fig:us_crime_num_edges} returns a graph with around 60 edges that passes Markov; there are models with fewer edges that look to be preferable, and the CAFS model has among the fewest number of edges for the starred models. The BIC plot bears this out; the CAFS model has among the highest scores. Some models have higher BIC scores, suggesting that BIC could also be used as a principle to choose among the starred models, though the sample size is small, suggesting that slightly higher BIC scores may not indicate better models. This is further borne out by the plot for CFCI and NFI. We can look up the CAFS model here; It turns out to be a model with 19 edges given by FGES; Three parameterizations of FGES yield the same model ($\lambda$ = 1.75, 2.0, and 2.50); this has a CFI of 0.9443 and an NFI of 0.8597. We show this model in Figure \ref{fig:us_crime_fges_penalty_2}. Since there are other models passing Markov that are close to the CAFS model; these might be considered plausible alternatives, and these models could also be looked up. It is useful to note that despite the small sample size, there are many models rejected by the Markov Checker, suggesting that the Markov Checker is effective in rejecting less plausible models.

\begin{figure}
    \centering
    \includegraphics[width=0.7\linewidth]{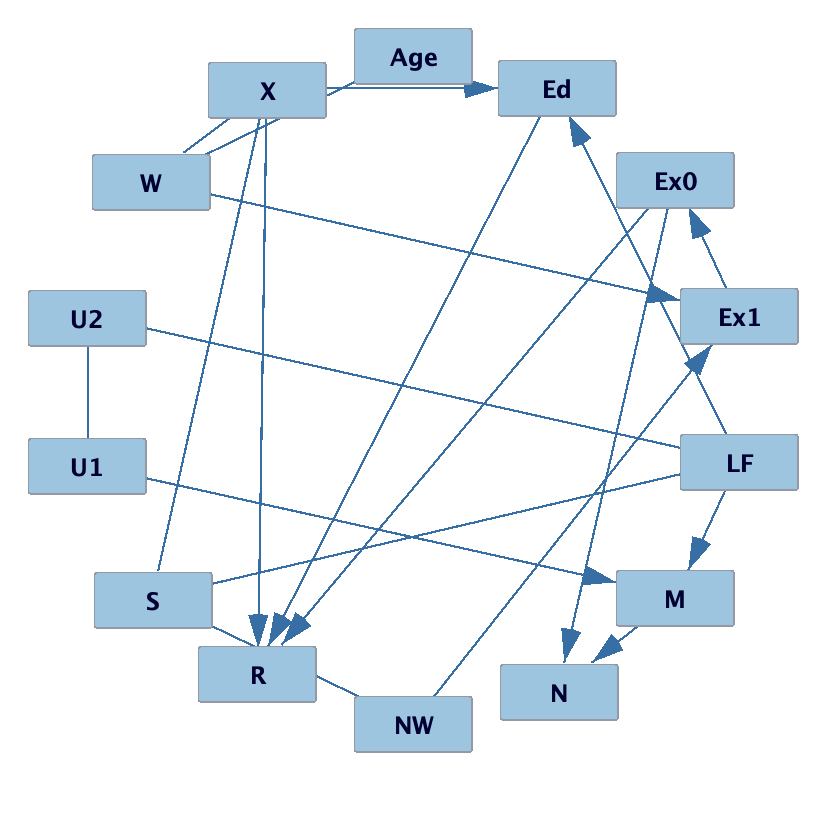}
    \caption{The result of applying the FGES algorithms with $\lambda = 2$ to the US Crime data. The rationale for selecting this model from among the models produced by various algorithms is given in the text.}
    \label{fig:us_crime_fges_penalty_2}
\end{figure}

\section{Conclusion}

DAG models make numerous claims about the independence and conditional independence of variables, so checking Markov for such models given a dataset is something that a domain expert can do without reference to ground truth. If one does not have enough conditional independence p-values to get a good Markov check test, one can generate more by calculating the p-values for these conditional independence tests for replacement random subsets, although, as mentioned in Section \ref{sec:overlap}, there is a trade-off of the number of additional p-values generated to the ability to detect dependencies. 

We give a Markov test based on the uniformity of the p-values obtained by such conditional independence tests. This uniformity can be explored across various algorithms with various tuning parameter settings; combinations that do not satisfy Markov with minimal edge counts can be rejected without knowing the ground truth as not being in a correct MEC, leaving a smaller set of plausible models for researchers to consider. Combined with Raskutti and Uhler's advice for looking for frugal models, this expands the toolkit for domain researchers, and it provides a tool that can be applied even to large models, which model p-value cannot, and it gives a ``thumb-up''/``thumbs-down'' assessment for models, which BIC or model scores cannot, being relative assessments.

The success of CAFS is not guaranteed. First, frugal or nearly frugal models need to be in the list of models inferred by algorithms. As a result, for very dense models, unless very good models exist in the list of available models, models may be accepted with lower precision and recall for adjacencies and arrowheads than desired. In our simulations, BiDAG and some BOSS and GRaSP results for the linear Gaussian case is that CAFS sometimes rejects the absolute best models by accuracy as having slightly too many edges compared to models frugally passing a Markov check, suggesting that at finite sample sizes the leaving out of weak edges cannot be detected due to insufficient power. Nevertheless, the models that CAFS accepts are generally quite good as long as algorithms are included in the comparison that yield frugal CPDAGs. In addition, the Markov check eliminates many models with inferior results, which means that the Markov check is doing its job.

Further exploration is needed. Several such points are noted above, but we may add at least the following. First, checking the Markov condition in the way we recommend requires a fast conditional independence test; general tests of conditional independence are still cubic in sample size, so they are not suitable for quick checking of the uniformity of p-values, which one can do in other cases. A remedy to this is to find ways to parallelize these tests with large machines; other possible remedies are noted above. Second, frugality may not be the only assumption one wishes to test for empirical data. For example, if one assumes linearity and non-Gaussianity, edges that cannot be oriented in a CPDAG may nevertheless be orientable, and similarly of the data are nonlinear or heteroskedastic. Third, the fact that we can pick out particularly good results from a list of algorithms and algorithm parameterizations suggests that other CPDAG algorithms in the literature may also succeed in producing frugal models; further exploration is warranted. For this reason, it would make sense to include this check in causal software-checking tools. Finally, the Markov checking software can be expanded to non-DAG models, e.g. latent variable models such as Partial Ancestral Graphs (PAGs), \cite{spirtes2000causation}.

\section{Acknowledgements}

We are thankful for comments on the content from Wai-Yin Lam and Clark Glymour. The original request for a check of the Markov condition was due to Michael Konrad, who has offered many useful comments as well. Teddy Seidenfeld gave us reference to the proof in \citep{casella2021statistical} that the p values for all distributions with continuous CDFs (not just invertible ones) are distributed as U(0, 1).

JR was supported by the US Department of Defense under Contract Number FA8702-15-D-0002 with Carnegie Mellon University for the operation of the Software Engineering Institute. BA was supported by the US National Institutes of Health under the Comorbidity: Substance Use Disorders and Other Psychiatric Conditions Training Program T32DA037183. PS was supported by the National Institutes of Health (NIH) under Contract R01HL159805, and grants from Apple Inc., and KDDI Research Inc.

\bibliography{refs}

\end{document}